%% file: main.tex
\pdfoutput=1

\documentclass[11pt]{article}

\usepackage[]{acl}

\usepackage{times}
\usepackage{latexsym}

\usepackage[T1]{fontenc}

\usepackage[utf8]{inputenc}

\usepackage{microtype}

\usepackage{inconsolata}

\usepackage{amsmath}
\usepackage{amsthm}
\usepackage{amssymb}
\usepackage{enumitem}
\newtheorem{problem}{Problem}
\newtheorem{lemma}{Lemma}
\newtheorem{theorem}{Theorem}

\newtheorem{proposition}{Proposition}
\usepackage{graphicx}
\usepackage{algorithm}
\usepackage{algpseudocode}
\usepackage{bbding}
\usepackage{pifont} 

\usepackage{booktabs}
\usepackage{multirow}

\newcommand{\Xmark}{\ding{55}} 

\title{Understanding Cross-Domain Adaptation in Low-Resource Topic Modeling}

\author{Pritom Saha Akash $\quad$ Kevin Chen-Chuan Chang \\
University of Illinois at Urbana-Champaign, USA \\
 \texttt{\{pakash2, kcchang\}@illinois.edu}
}

\begin{document}
\maketitle
\input{sections/abstract}
\input{sections/introduction}

\input{sections/related_work}

\input{sections/proposed_model}
\input{sections/experiments}
\input{sections/conclusion}
\input{sections/limitation}
\input{sections/acknowledgements}

\bibliography{custom}

\appendix
\input{sections/appendix}

\end{document}

%% file: sections/abstract.tex
\begin{abstract}
Topic modeling plays a vital role in uncovering hidden semantic structures within text corpora, but existing models struggle in low-resource settings where limited target-domain data leads to unstable and incoherent topic inference. We address this challenge by formally introducing domain adaptation for low-resource topic modeling, where a high-resource source domain informs a low-resource target domain without overwhelming it with irrelevant content. We establish a finite-sample generalization bound showing that effective knowledge transfer depends on robust performance in both domains, minimizing latent-space discrepancy, and preventing overfitting to the data. Guided by these insights, we propose DALTA (Domain-Aligned Latent Topic Adaptation), a new framework that employs a shared encoder for domain-invariant features, specialized decoders for domain-specific nuances, and adversarial alignment to selectively transfer relevant information. Experiments on diverse low-resource datasets demonstrate that DALTA consistently outperforms state-of-the-art methods in terms of topic coherence, stability, and transferability.
\end{abstract}

%% file: sections/introduction.tex
\section{Introduction}
\label{sec:introduction}


In today’s digital age, large volumes of unstructured text corpora are produced across various domains, making it essential to derive meaningful insights. \textbf{Topic modeling} helps uncover hidden patterns in this text, enabling applications such as document classification, text summarization, content recommendation, and trend analysis. While probabilistic topic models \cite{blei2003latent,blei2006correlated,blei2006dynamic,mcauliffe2007supervised} remain widely used, deep learning has driven the emergence of more advanced variants. For instance, Neural Topic Models (NTMs) \cite{miao2016neural,srivastava2017autoencoding,nguyen2021contrastive,dieng2020topic} leverage deep learning to improve topic representations, often employing Variational Autoencoders (VAEs) to model latent spaces. Further advancements, such as Contextualized Topic Models (CTMs) \cite{bianchi2020pre,bianchi2020cross,grootendorst2022bertopic}, integrate pre-trained language models to enhance topic coherence by capturing contextual dependencies.

Despite their success, topic models typically assume the availability of sufficient data to learn meaningful and coherent topics. However, in many real-world scenarios—particularly in emerging or niche domains—data collection is limited by resource constraints, restricted access, or the rapid evolution of knowledge. For instance, fields like \textit{quantum machine learning} may have fewer than 1,000 publicly available documents, while specialized domains such as \textit{legal} or \textit{medical} texts are subject to strict privacy regulations. In such cases, conventional topic models struggle to extract stable and coherent topics from low-resource corpora.

\begin{figure}[!tb]
\centering
\includegraphics[width=1.0\linewidth]{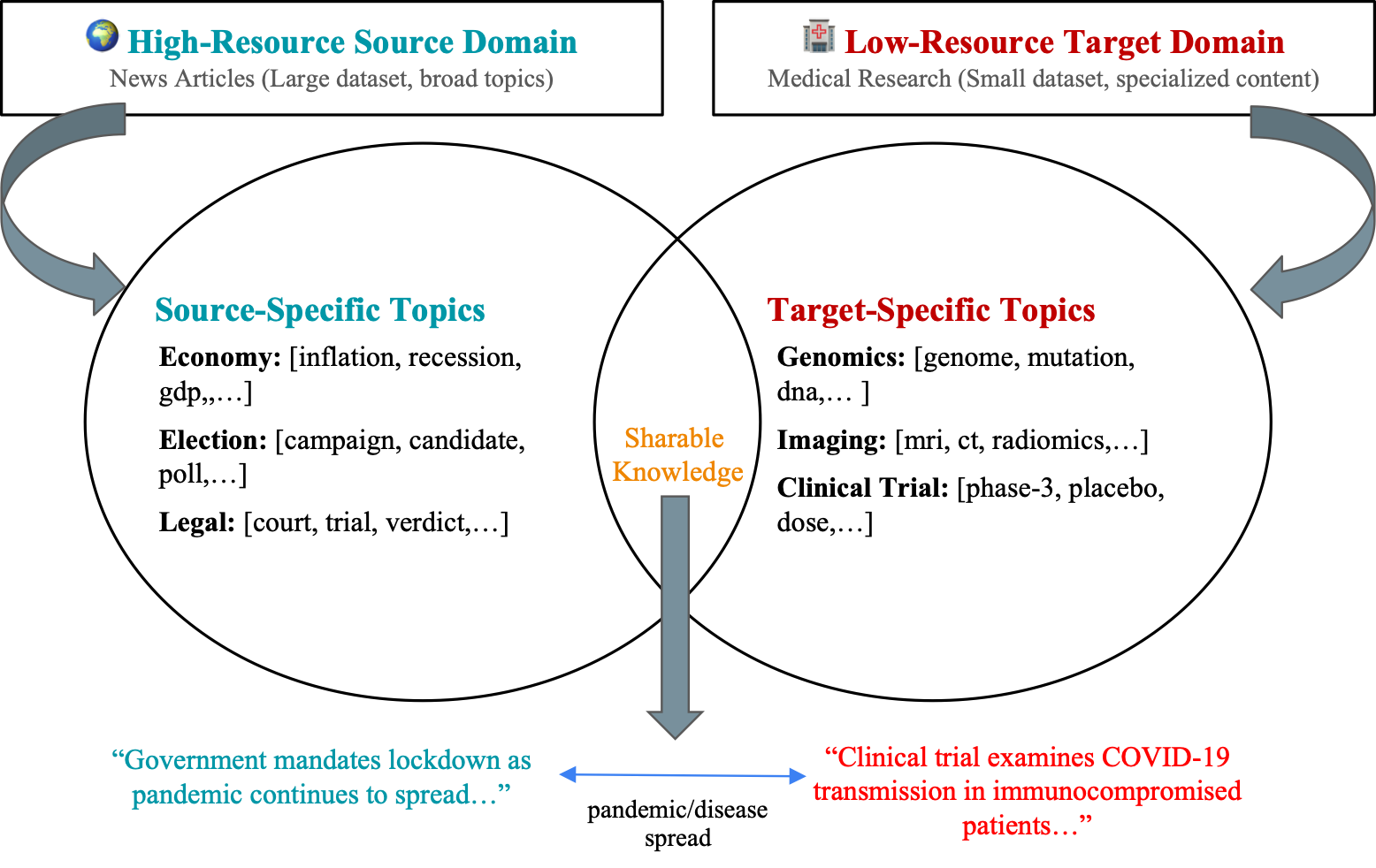}
\caption{An illustration of how a high-resource domain (news) and a low-resource domain (medical research) share certain topics (e.g., “pandemic/disease spread”) while each retains domain-specific content (e.g., “economy” vs. “genomics”). The goal is to transfer only the relevant knowledge without introducing irrelevant information.}
\label{fig:example}
\end{figure}

While there have been several attempts to address \textbf{low-resource topic modeling}, they exhibit some limitations. For instance, an earlier approach \cite{sia2021adaptive} refines LDA by adaptively balancing discrete token statistics with pretrained word embeddings, allowing frequent words to rely on counts while infrequent words leverage external representations. Similarly, embedding-based NTMs \cite{duan2022bayesian,duan2021sawtooth} leverage pretrained word embeddings as transferable knowledge, enabling effective generalization by learning topic embeddings adaptively. However, as word semantics shift across contexts, static embeddings may fail to adapt to unseen tasks, limiting their effectiveness in domain-specific low-resource settings. To address this, Context-Guided Embedding Adaptation \cite{xu2024context} dynamically refines topic-word relationships by adjusting word embeddings based on syntactic and semantic dependencies in the target corpus. While this improves topic coherence in low-resource settings, it still depends solely on target domain data, making it less effective when extreme data scarcity limits meaningful adaptation.

From the above discussion, the key challenge in low-resource topic modeling is effectively leveraging external knowledge while preserving target-specific nuances. As illustrated in Figure~\ref{fig:example}, consider a high-resource source domain like news articles and a low-resource target domain such as medical research. The source offers a broad range of topics—some, like pandemics and mental health, provide shareable knowledge (here shown as pandemic/disease spread), while others, such as elections or economics, remain source-specific and irrelevant to medical research. Similarly, the target domain contains specialized concepts (e.g., genomics, clinical trials), which do not appear in the source. By focusing on pandemic/disease spread, news articles (e.g., “Government mandates lockdown…”) can inform medical studies (e.g., “Clinical trial examines COVID-19 transmission…”), allowing the model to transfer relevant health-related content. However, topics like election or economy from the source domain have no bearing on genomics or imaging in the target. The key challenge, therefore, is to maximize the shareable topics while ensuring that domain-specific information remains intact, thereby enhancing topic discovery under data-scarce conditions.

Building on these insights, \textbf{domain adaptation} \cite{ben2010theory} offers a promising framework to bridge the gap between high-resource source data and low-resource target needs in topic modeling. Although domain adaptation has been widely used in supervised tasks to align features across domains \cite{zhao2019learning,li2021learning}, its potential in unsupervised topic modeling remains largely unexplored. To address this, we formalize domain adaptation for topic modeling by establishing a generalization bound that shows successful transfer depends on: (i) achieving strong performance on both source and target data, (ii) minimizing the discrepancy between source and target latent representations so that only relevant knowledge is transferred, and (iii) applying regularization to prevent overfitting to the source domain. In light of this, we introduce our novel bound minimization algorithm, \textbf{DALTA} (Domain-Aligned Latent Topic Adaptation), which leverages these principles to selectively transfer useful topic structures from high-resource domains while preserving the unique semantic characteristics of the target domain.

In summary, our work provides the following contributions:
\begin{itemize}[nolistsep,leftmargin=*]
\item We derive \textbf{finite-sample generalization bounds} for domain adaptation in low-resource topic modeling, demonstrating that effective transfer relies on robust performance in both source and target domains, alignment of latent representations, and proper regularization to prevent overfitting.
\item Building on these theoretical insights, we propose \textbf{DALTA (Domain-Aligned Latent Topic Adaptation)}, a novel framework that employs a shared encoder to extract domain-invariant topic representations and specialized decoders to capture target-specific semantic nuances. To our knowledge, DALTA is the first method to jointly optimize latent alignment and domain-specific learning in topic modeling under a rigorous theoretical foundation.
\item We conduct extensive experiments on diverse low-resource datasets, showing that DALTA consistently outperforms state-of-the-art methods in terms of topic coherence, diversity, and transferability.
\end{itemize}

%% file: sections/related_work.tex
\section{Related Work}
\label{sec:related_work}

\subsection{Neural Topic Models}
Traditional probabilistic topic models, such as Latent Dirichlet Allocation (LDA) \cite{blei2003latent} and its extensions \cite{blei2006correlated,blei2006dynamic,mcauliffe2007supervised}, have been widely used for discovering latent semantic structures in text. However, they rely on bag-of-words assumptions that ignore word order and contextual meaning, limiting their ability to capture nuanced semantics. To overcome these limitations, Neural Topic Models (NTMs) leverage deep learning, particularly Variational Autoencoders (VAEs) \cite{kingma2013auto}, to learn richer and more flexible topic representations \cite{miao2016neural,srivastava2017autoencoding}. Contextualized Topic Models (CTMs) \cite{bianchi2020pre, bianchi2020cross,grootendorst2022bertopic} further improve topic coherence by incorporating pre-trained language models (PLMs), allowing them to capture contextual dependencies. More recent works, such as UTopic \cite{han2023unified}, introduce contrastive learning and term weighting to enhance topic coherence and diversity, effectively refining topic representations. Similarly, NeuroMax \cite{pham2024neuromax} enhances NTMs by maximizing mutual information between topics and enforcing structured topic regularization to improve coherence.

Despite these advancements, NTMs still face challenges when applied in cross-domain or low-resource settings. Most models assume an abundance of training data and struggle with domain adaptation, where vocabulary shifts and distributional changes lead to misaligned topic representations. Although recent works, such as prompt-based NTMs \cite{pham2023topicgpt} and LLM-driven context expansion topic models \cite{akash2024enhancing}, attempt to leverage external knowledge sources for more robust topic discovery, they do not explicitly address how to transfer topic knowledge across domains with varying data distributions. As a result, current NTMs often fail to generalize effectively beyond their training domains, necessitating new approaches that can balance domain-invariant knowledge transfer with domain-specific adaptability.

\subsection{Low-Resource Topic Modeling}
Low-resource topic modeling has been explored through meta-learning and embedding-based adaptations, each with inherent limitations. Few-shot approaches \cite{iwata2021few} attempt to learn task-specific priors for generalization from limited samples.However, their rigid probabilistic assumptions prevent them from effectively capturing the contextual variations and subtle topic differences that arise in new domains. Embedding-based Neural Topic Models (NTMs) \cite{duan2022bayesian,duan2021sawtooth} improve generalization by leveraging pre-trained word embeddings, allowing topic discovery beyond simple word co-occurrence patterns. However, their reliance on static representations limits their adaptability in domains where word semantics shift significantly across contexts.

To further address data scarcity, Meta-CETM \cite{xu2024context} adapts pre-trained contextual embeddings to improve topic modeling in low-resource domains. However, it does not explicitly align topic distributions between source and target domains, relying instead on implicit adaptation through embedding refinements. This approach assumes that target-domain context provides sufficient information for effective adaptation, but in domains with sparse or highly specialized terminology, the model can overfit to limited contextual signals, leading to unstable topic representations and poor generalization. On the other hand, FASTopic \cite{wu2024fastopic} adopts a fully pre-trained transformer-based topic model, bypassing the need for fine-tuning on target-domain data. While this approach improves efficiency and avoids overfitting to small target datasets, it assumes that source-domain knowledge is universally applicable.

%% file: sections/proposed_model.tex
\section{Proposed Methodology}
\label{sec:proposed_model}

In this section, we formally define the problem of domain adaptation for low-resource topic modeling and establish the theoretical foundation for our approach. We first derive generalization bounds to quantify the conditions for effective knowledge transfer from a high-resource source domain to a low-resource target domain. Leveraging these insights, we then introduce DALTA (Domain-Aligned Latent Topic Adaptation), which aligns latent representations while preserving domain-specific structure.

\begin{problem}[Domain Adaptation for Low-resource Topic Modeling]
Let $\mathcal{X} \subseteq \mathbb{R}^d$ represent a $d$-dimensional document data space, where $\mathcal{X} = \{\mathbf{x}_1, \mathbf{x}_2, \ldots, \mathbf{x}_n\}$ is the set of document representations (e.g., bag-of-words and/or embeddings) with marginal distribution $p(\mathcal{X})$. The \textbf{source domain} is defined as $(\mathcal{X}_S, p(\mathcal{X}_S))$, and the \textbf{target domain} is defined as $(\mathcal{X}_T, p(\mathcal{X}_T))$, where $\mathcal{X}_S \neq \mathcal{X}_T$ (e.g., different vocabularies or structures) and $p(\mathcal{X}_S) \neq p(\mathcal{X}_T)$ (e.g., distributional shifts). The topic spaces $\alpha_S$ and $\alpha_T$ represent the domain-specific latent topic proportions for documents in the source and target domains, respectively, while the topic-word distributions $\beta_S$ and $\beta_T$ capture the relationship between topics and words in each domain. Domain adaptation seeks to infer meaningful topics for the low-resourced target domain $\mathcal{X}_T$ by leveraging knowledge from the source domain $\mathcal{X}_S$, while addressing challenges such as vocabulary mismatches, topic shifts, and distributional differences between the domains.
\end{problem}

\subsection{Generalization Bound}


To better understand the challenges of domain adaption in low-resource topic modeling, consider a scenario where we aim to analyze documents from two distinct domains: computer science (source domain) and medical science (target domain). While the source domain provides an abundance of documents, the target domain suffers from data scarcity. Our goal is to develop a topic model that generalizes well to the target domain despite the limited availability of data. This scenario raises fundamental questions: can we effectively leverage the abundant source domain data to improve topic modeling in the target domain, and how can we theoretically guarantee the effectiveness of such transfer?

To address these questions, we derive a finite-sample generalization bound for domain adaptation in topic modeling, offering insights into the factors that influence knowledge transfer between domains. Our theoretical framework is grounded in the concept of generalization error, which measures how well a model trained on finite samples can perform on unseen data. Specifically, we aim to bound the target domain error $\epsilon_T(h)$ in terms of observable quantities from both the source and target domains.


\begin{theorem} [Generalization bound]
\label{theorem1}

Let $h \in \mathcal{H}$ be a hypothesis from a hypothesis class $\mathcal{H}$, where $h: \mathcal{Z} \to [0,1]^{|\mathcal{V}|}$ maps from a latent semantic space $\mathcal{Z}$ to a probability distribution over a vocabulary $\mathcal{V}$. Let $f_S$ and $f_T$ be the optimal functions mapping latent representations to reconstructed outputs for the source and target domains, respectively. Define $p_S = \frac{n_S}{n_S + n_T}$ as the proportion of source samples and $p_T = \frac{n_T}{n_S + n_T}$ as the proportion of target samples. Then, for every $h \in \mathcal{H}$ and for any $\delta > 0$, with probability at least $1 - \delta$ over the choice of the source and target samples of sizes $n_S$ and $n_T$, the target domain error is bounded by:

\resizebox{\columnwidth}{!}{%
\begin{minipage}{\columnwidth}
    \begin{align*}
    \epsilon_T(h) &\leq p_T\cdot\hat{\epsilon}_T(h) + p_S\cdot\hat{\epsilon}_S(h) \\&+ \frac{1}{\lambda}\text{KL}(q \| p)
     +p_S\cdot(d_\mathcal{H}(\mathcal{D}_S(Z), \mathcal{D}_T(Z)) \\&+ p_S\cdot\min \left\{ \mathbb{E}_{S} \big[| f_S - f_T | \big], \, \mathbb{E}_{T} \big[ | f_S - f_T | \big] \right\}  \\&+
     \mathcal{O}\Big(K_\phi K_\theta \Delta + \frac{1}{\lambda} \log \frac{1}{\delta} + \frac{\lambda \Delta^2}{n_S + n_T}\Big)
    \end{align*}
\end{minipage}%
}
\end{theorem}

The generalization bound from Theorem \ref{theorem1} highlights key factors governing domain adaptation in topic modeling. The first two terms capture empirical reconstruction errors in the source and target domains, emphasizing the need for latent topic representations that effectively reconstruct documents across both. The third term, the KL divergence $\text{KL}(q \| p)$, regularizes the model by constraining learned latent representations to remain close to the prior, thus reducing overfitting. The fourth term measures the discrepancy between source and target latent representations, with lower values indicating better domain alignment and improved transferability of learned topics. The fifth term quantifies the divergence between optimal reconstruction functions, reflecting how well topic-word distributions align across domains. Finally, the last term, involving higher-order complexity measures, accounts for the model’s capacity and the statistical fluctuations due to finite sample sizes.

\begin{figure}[!tb]
\centering
\includegraphics[width=0.9\linewidth]{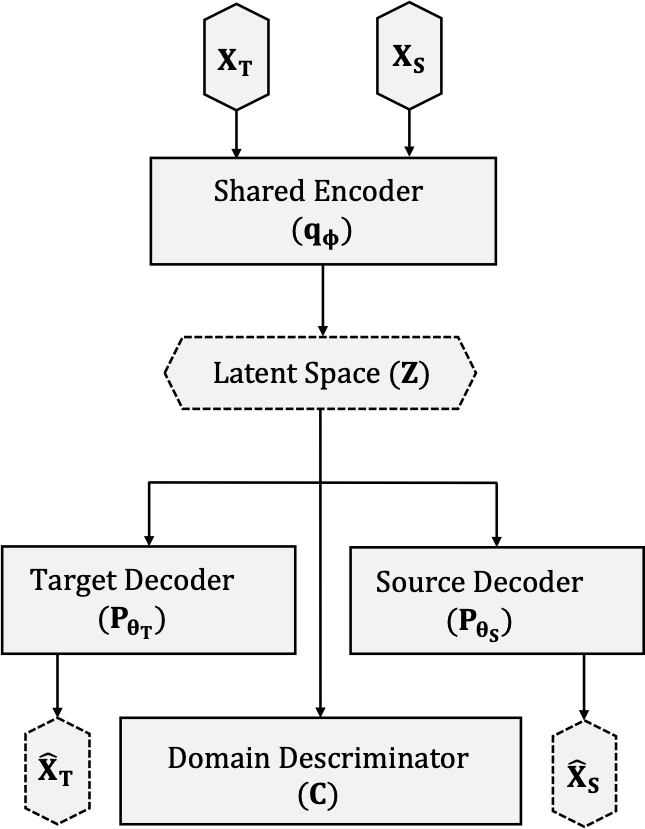}
\caption{DALTA Framework}
\label{fig:dalta}
\vspace{-5mm}
\end{figure}

\subsection{Proposed Model: DALTA}

Motivated by the generalization error bounds in Theorem \ref{theorem1}, we introduce DALTA (Domain-Aligned Latent Topic Adaptation), a bound minimization framework (shown in Figure~\ref{fig:dalta}) aimed at improving domain adaptation in low-resource topic modeling. The key idea behind DALTA is to focus on optimizing the components of the bound that directly influence model performance in the target domain. Since the last term of the bound reflects model complexity and sample variance—factors that are less tractable—we prioritize minimizing the first five terms, which capture reconstruction quality, representation alignment, and regularization.

To achieve this, DALTA minimizes the empirical reconstruction errors using source and target data, promoting accurate document reconstruction. It encourages alignment between source and target latent spaces to reduce domain discrepancies, and fosters consistency in reconstruction functions to ensure similar performance across domains. Additionally, the KL divergence term is employed as a regularizer to maintain a stable latent space structure, supporting generalization while mitigating overfitting. This targeted optimization strategy enables DALTA to effectively adapt topic models across diverse domains.

To learn domain-invariant information, we employ a shared encoder $q_\phi: \mathcal{X} \to \mathcal{Z}$  that maps documents $X$ from both source and target domains into a common latent space $Z$. To achieve domain-invariant representations, the encoder aims to minimize the discrepancy between the source and target latent distributions. We adopt an adversarial training approach as \cite{ganin2016domain}, introducing a domain discriminator $C$ that distinguishes between source and target representations, while the encoder $q_\phi$  is optimized to fool $C$. The adversarial objective is formulated as:

\begin{align}
\label{eq:l_adv}
\min_{q_\phi} \max_{C} \, \mathcal{L}_{adv} = \mathbb{E}_{\mathcal{X}_S}[\log C(q_\phi(X_S))] \nonumber \\ + \mathbb{E}_{\mathcal{X}_T}[\log(1 - C(q_\phi(X_T)))]    
\end{align}

This min-max optimization ensures that the encoder learns domain-invariant features by reducing the discriminator’s ability to differentiate between the two domains. To formally characterize this relationship, we express the connection between the domain classifier’s performance and the divergence between the source and target domains as follows:

\begin{proposition}
Let \( q_\phi: \mathcal{X} \to \mathcal{Z} \) be a shared encoder mapping documents from the source domain \( \mathcal{X}_S \) and target domain \( \mathcal{X}_T \) into a common latent space \( \mathcal{Z} \). The \(\mathcal{H}\)-divergence between the source and target latent distributions is given by:
\[
d_{\mathcal{H}}(q_\phi(\mathcal{X}_S), q_\phi(\mathcal{X}_T)) = 2 \left(1 - 2 \epsilon_C^* \right),
\]
where \( \epsilon_C^* \) is the classification error of the optimal domain classifier \( C^* \).
\end{proposition}

As  $\epsilon_{C} \to 0.5$  (random guessing by  $C^*$),  $d_{\mathcal{H}}(q_\phi(\mathcal{X}_S), q_\phi(\mathcal{X}_T)) \to 0$ , indicating perfect alignment between  $q_\phi(\mathcal{X}_S)$  and  $q_\phi(\mathcal{X}_T)$. This process directly contributes to minimizing the fourth term in the generalization bound.

To capture domain-specific characteristics, DALTA incorporates decoders  \( p_{\theta_S}(X_S|Z) \) and \( p_{\theta_T}(X_T|Z) \). Each decoder infers document-topic distributions— $\alpha_S$ and $\alpha_T$ for the source and target domains—by mapping to topic proportions, which are then used to reconstruct documents. This intermediate step ensures that latent representations capture domain-relevant semantics. The number of topics can vary between domains and is independent of the latent space size, providing flexibility to handle diverse topic granularities. The reconstruction objective is defined as:
\begin{align}
\label{eq:l_rec}
\min_{q_\phi, p_{\theta_S}, p_{\theta_T}} \mathcal{L}_{rec} = - \mathbb{E}_{q_\phi(Z|X_S)}[\log p_{\theta_S}(X_S|Z)] \nonumber \\ - \mathbb{E}_{q_\phi(Z|X_T)}[\log p_{\theta_T}(X_T|Z)],
\end{align}
which corresponds to minimizing the first two terms in the generalization bound.

To enhance the consistency of learned representations across domains, DALTA introduces a consistency loss that reduces the divergence between the optimal reconstruction functions of the source and target decoders. This is achieved by enforcing similarity between the outputs of the domain-specific decoders when processing aligned latent representations:
\resizebox{\columnwidth}{!}{%
\begin{minipage}{\columnwidth}
\begin{align}
\label{eq:l_cons}
\min_{q_\phi, p_{\theta_S}, p_{\theta_T}} &\mathcal{L}_{cons} = \mathbb{E}_{q_\phi(Z|X_S)} [\| p_{\theta_S}(Z) - p_{\theta_T}(Z) \|^2] \nonumber \\ & + \mathbb{E}_{q_\phi(Z|X_T)} [\| p_{\theta_S}(Z) - p_{\theta_T}(Z) \|^2],
\end{align}
\end{minipage}%
}
which directly contributes to minimizing the fifth term in the generalization bound by promoting functional alignment between source and target domains.

To prevent overfitting and maintain a smooth latent space structure, DALTA employs a KL divergence regularizer:

\begin{align}
\label{eq:l_kl}
\min_{q_\phi} \mathcal{L}_{KL} = D_{KL}(q_\phi(Z|X) \| p(Z)),    
\end{align}
where \( p(Z) \) is typically chosen as a standard Gaussian prior. This regularization helps control the model's complexity, addressing the third term of the generalization bound and ensuring robust generalization.

Bringing these components together, the overall objective of DALTA is formulated as:
\begin{align}
\label{eq:l_tot}
\min_{q_\phi, p_{\theta_S}, p_{\theta_T}} \max_{C} \mathcal{L}_{DALTA} = \mathcal{L}_{rec} + \omega_{adv} \mathcal{L}_{adv} \nonumber \\ + \omega_{cons} \mathcal{L}_{cons} + \omega_{KL} \mathcal{L}_{KL},    
\end{align}
where \( \omega_{adv} \), \( \omega_{cons} \), and \( \omega_{KL} \) are hyperparameters that balance the contributions of the adversarial, consistency, and regularization losses. This integrated framework enables DALTA to effectively balance domain-invariant and domain-specific learning, facilitating robust and adaptable cross-domain topic modeling. The detailed training procedure of DALTA is outlined in Algorithm~\ref{alg:dalta}.

\begin{algorithm}[H]
\caption{Learning DALTA}
\label{alg:dalta}
\begin{algorithmic}[1]
\Require Source domain data $\mathcal{X}_S$, target domain data $\mathcal{X}_T$, learning rates, domain-weight parameter $\mu$, trade-off parameters $\omega_{adv}, \omega_{cons}, \omega_{KL}$
\State Initialize encoder parameters $\phi$, decoder parameters $\theta_S, \theta_T$, and domain discriminator $C$
\While{not converged}
    \State Sample mini-batch from $\mathcal{X}_S$ and $\mathcal{X}_T$
    \State Encode documents: $Z_S = q_\phi(X_S)$, $Z_T = q_\phi(X_T)$
    \State Compute losses: $\mathcal{L}_{rec}$ \eqref{eq:l_rec}, $\mathcal{L}_{adv}$ \eqref{eq:l_adv}, $\mathcal{L}_{cons}$ \eqref{eq:l_cons} and $\mathcal{L}_{KL}$ \eqref{eq:l_kl}
    \State Compute total loss: $\mathcal{L}_{DALTA}$ \eqref{eq:l_tot}
    \State Update $\phi, \theta_S, \theta_T$ to minimize $\mathcal{L}_{DALTA}$
    \State Update discriminator $C$ to maximize $\mathcal{L}_{adv}$
\EndWhile
\Return Optimized parameters $\phi, \theta_S, \theta_T, C$
\end{algorithmic}
\end{algorithm}

%% file: sections/experiments.tex
\section{Experiments}
\label{sec:experiment}
In this section, we conduct a comprehensive set of experiments in low-resource settings to evaluate topic quality, classification accuracy, document clustering performance, and the impact of individual loss components. We also include a qualitative analysis of topic interpretability and a case study on source domain selection, with extended results provided in the appendix.

\subsection{Experiment Setup}

{\flushleft \textbf{Datasets.}} 
We evaluate cross-domain adaptation in low-resource topic modeling using four diverse target datasets, each sampled to 1,000 instances per domain: (1) \textbf{20 Newsgroups}\footnote{\url{https://www.kaggle.com/datasets/crawford/20-newsgroups}}, where we use \textit{Science} and \textit{Religion} to assess adaptation between technical and belief-based topics; (2) \textbf{Drug Review}\footnote{\url{https://www.kaggle.com/datasets/jessicali9530/kuc-hackathon-winter-2018}}, consisting of patient reviews on two drugs named \textit{Norethindrone} and \textit{Norgestimate} to evaluate adaptation within medical text; (3) Yelp Reviews\footnote{\url{https://www.kaggle.com/datasets/omkarsabnis/yelp-reviews-dataset}}, representing informal, sentiment-rich business reviews; and (4) \textbf{SMS Spam Collection}\footnote{\url{https://archive.ics.uci.edu/dataset/228/sms+spam+collection}}, containing labeled \textit{spam} and \textit{ham} messages to test adaptation in short-text data with high lexical variability.

For the \textbf{source dataset}, we use the \textbf{AG News corpus}\footnote{\url{https://www.kaggle.com/datasets/amananandrai/ag-news-classification-dataset}}, a large-scale news dataset covering \textit{World}, \textit{Sports}, \textit{Business}, and \textit{Science/Technology} topics. As a high-resource domain, it provides broad topic coverage, allowing models to learn transferable topic representations for adaptation to low-resource target domains.

{\flushleft \textbf{Baselines.}}
We compare our models with several established baselines: (1) \textbf{LDA} \cite{blei2003latent} models documents as mixtures of topics, each represented by a distribution over words. (2) \textbf{ProdLDA} \cite{srivastava2017autoencoding} employs variational autoencoders to infer document-topic distributions. (3) \textbf{ETM} \cite{dieng2020topic} incorporates word embeddings to enhance topic coherence. (4) \textbf{CTM} \cite{bianchi2020pre} integrates contextualized document embeddings with bag-of-words representations. (5) \textbf{ECRTM} \cite{wu2023effective} enforces distinct word embedding clusters for each topic to prevent topic collapse. (6) \textbf{DeTiME} \cite{xu-etal-2023-detime} utilizes encoder-decoder-based large language models to generate semantically coherent topic embeddings. (7) \textbf{Meta-CETM} \cite{xu2024context} adapts word embeddings using target domain context for low-resource settings. (8) \textbf{FASTopic} \cite{wu2024fastopic} models semantic relations among documents, topics, and words through a dual semantic-relation reconstruction paradigm.

For implementation details and computing infrastructure, please refer to Appendices \ref{sec:imp_details} and \ref{sec:compute}.

\subsection{Topic Quality Evaluation} 

\input{tables/coherence_ag_1000samples}

{\flushleft \textbf{Evaluation Metrics.}}
For evaluating the quality of topics returned by each model, we use the following two different metrics-- (1) \textbf{$C_V$} \cite{wu2020short}: We use the widely used coherence score for topic modeling named $C_V$. It is a standard measure of the interpretability of topics. (2) \textbf{$TD$} \cite{nan2019topic}: Topic diversity (TD),
defined as the percentage of unique words in the top 10 words of all topics.

{\flushleft \textbf{Results and Discussions.}} Table \ref{tab:coherence_ag_1000samples} presents the topic coherence ($C_V$) and topic diversity ($TD$) scores for various models across multiple datasets and topic numbers. Higher coherence scores indicate better semantic consistency within topics, while higher diversity scores reflect broader coverage of unique words across topics. Our proposed DALTA model consistently achieves the highest coherence and diversity scores in almost all settings, demonstrating its ability to generate both semantically meaningful and diverse topics. Notably, DALTA outperforms all baselines in coherence across every dataset and setting, with particularly strong improvements in specific domains like Drug Review and Spam Collection. This suggests that DALTA effectively balances domain-invariant knowledge transfer while preserving domain-specific topic structures.

When comparing other models, ETM and CTM generally outperform LDA and ProdLDA by leveraging word embeddings and contextualized representations, which improve coherence. However, this often comes at the cost of topic diversity, leading to less comprehensive topic coverage. Meta-CETM and FASTopic perform well in more general datasets like Yelp, but struggle in niche settings, where DALTA proves to be more stable and robust. One interesting trend is that increasing the number of topics  ($k=10$ to $k=20$) tends to improve diversity, but does not always enhance coherence. For example, in datasets like Newsgroup and Yelp, raising the topic count does not necessarily result in more coherent topics. Despite this, DALTA maintains an optimal balance between coherence and diversity, making it particularly well-suited for low-resource topic modeling scenarios.

\subsection{Text Classification Evaluation}

Although topic models are not primarily designed for text classification, the document-topic distributions they generate can serve as useful features for classification tasks. To assess how well these representations capture meaningful document characteristics, we use them as input features for Support Vector Classification (SVC) \cite{cortes1995support} and Logistic Regression (LR) \cite{wright1995logistic}. We evaluate classification performance using 5-fold cross-validation, ensuring a robust comparison of different topic models based on their ability to generate informative and distinctive document representations.

\input{tables/classification_ag_1000samples}

{\flushleft \textbf{Results and Discussions.}} Table \ref{tab:classification_ag_1000samples} presents the text classification performance. Similar to the topic quality, our proposed DALTA model achieves the highest classification accuracy in most cases, particularly in niche datasets, where precise topic separation is crucial. In Spam Collection, DALTA achieves 0.975 (SVC) and 0.978 (LR), outperforming Meta-CETM and FASTopic. Similarly, in Drug Review, DALTA achieves the best classification accuracy for both drug categories, Norethindrone and Norgestimate, indicating that it effectively captures domain-specific vocabulary for improved classification.

Among baseline models, FASTopic and Meta-CETM perform well in general datasets but show inconsistent performance in niche domains, likely due to their reliance on embedding-based adaptation rather than direct domain alignment. CTM and ECRTM benefit from increasing the number of topics, particularly in Newsgroup (e.g., Science, Religion), where finer topic granularity improves classification. LDA and ProdLDA show competitive accuracy in broad domains but struggle in niche settings, where more specialized topic representations are needed.

\subsection{Clustering Performance Evaluation}
\label{sec:clustering_result}
\input{tables/clustering_ag_1000samples}
Building on our classification findings, we now evaluate whether DALTA’s document–topic distributions naturally form coherent clusters without using any labels. We evaluate using two clustering metrics, Purity and Normalized Mutual Information (NMI) \cite{schutze2008introduction}, following Zhao et al. \cite{zhao2020neural}. Purity measures the proportion of correctly assigned documents within each inferred topic cluster, while NMI quantifies the mutual dependence between inferred and true topic assignments, providing insight into topic coherence and separation. Higher values in both metrics indicate better alignment between discovered topic structures and ground-truth labels.

Similar to the classification results, DALTA consistently outperforms baseline models in clustering, achieving the highest Purity and NMI scores across most datasets. Notably, DALTA performs exceptionally well on SMS Spam Collection (0.978 Purity, 0.287 NMI) and Drug Review (Norgestimate: 0.604 Purity, 0.071 NMI), highlighting its ability to enhance cluster quality in low-resource and specialized domains. Among baselines, CTM and Meta-CETM show competitive results on structured datasets like Newsgroup and Yelp, where CTM benefits from higher topic granularity. FASTopic performs well on Yelp, leveraging embedding-based adaptation for clustering. However, ETM and DeTiME exhibit weak NMI scores, suggesting difficulty in forming well-separated topic clusters. These results reinforce that DALTA’s domain-aware modeling improves both classification and clustering by learning more coherent and transferable topic representations.

\input{tables/ablation_study}
\input{tables/qualitative_result}
\subsection{Ablation Study} 
Table \ref{tab:ablation} evaluates the impact of different loss terms on topic quality and classification performance using the Newsgroup Science dataset. While this analysis is limited to a single dataset, it provides valuable insights into how each loss function contributes to \textbf{cross-domain topic adaptation}. The full \textbf{DALTA} model, incorporating all loss terms, achieves the best balance between coherence, diversity, and classification accuracy, confirming that each component plays a crucial role in optimizing topic modeling performance.

Removing $\mathcal{L}_{adv}$ lowers classification accuracy, emphasizing its role in domain alignment. Excluding $\mathcal{L}_{consist}$ reduces topic coherence, suggesting that consistency constraints help maintain stable and meaningful topic structures. The most significant drop in both coherence and classification accuracy occurs when $\mathcal{L}_{KL}$ is omitted, indicating that latent space regularization prevents topic collapse. Interestingly, topic diversity increases slightly without $\mathcal{L}_{KL}$, highlighting a trade-off between structural stability and diversity.

\subsection{Qualitative Topic Interpretability}

To further evaluate topic coherence and semantic clarity, we conduct a qualitative comparison across models using the \textbf{Yelp Reviews} dataset. For this analysis, we manually inspect two representative topics generated by each model and summarize them using the top five topic words. This evaluation allows us to assess not only the semantic distinctiveness of topics but also their practical interpretability.

As shown in Table~\ref{tab:qual-topic}, \textbf{DALTA} produces significantly more coherent and thematically consistent topics than other methods. Traditional models such as LDA and ProdLDA often yield vague or generic topics (e.g., \textit{good, food, come}), while more advanced models like CTM, Meta-CETM, and Fastopic exhibit signs of topic mixing or semantic drift. In contrast, DALTA generates sharper and more interpretable themes. For example, it identifies a focused food-related topic (\textit{taco, burrito, salsa, mexican, lunch}) and a distinct contextual theme related to social settings (\textit{kitchen, party, room, night, drink}). These findings underscore DALTA’s strength in generating concise and focused topics that better reflect the structure of review content, complementing its strong quantitative performance with practical interpretability in low-resource settings.

%% file: tables/coherence_ag_1000samples.tex
\begin{table*}[!ht]
\resizebox{1.0\linewidth}{!}{%
\begin{tabular}{@{}lllll|llll|llll|llll|llll|llll@{}}
\toprule
\multicolumn{1}{c}{}                        & \multicolumn{8}{c}{Newsgroup}                                                                                                                                                                                                              & \multicolumn{8}{c}{Drug Review}                                                                                                                                                                                           & \multicolumn{4}{c}{}                                                                                        & \multicolumn{4}{c}{}                                                                                        \\
\multicolumn{1}{c}{}                        & \multicolumn{4}{c}{Science}                                                                                 & \multicolumn{4}{c}{Religion}                                                                                                        & \multicolumn{4}{c}{Norethindrone}                                                                           & \multicolumn{4}{c}{Norgestimate}                                                                            & \multicolumn{4}{c}{\multirow{-2}{*}{Yelp}}                                                                  & \multicolumn{4}{c}{\multirow{-2}{*}{\begin{tabular}[c]{@{}c@{}} SMS Spam \\ Collection\end{tabular}}}            \\ \cmidrule(lr){2-25}
\multicolumn{1}{c}{}                        & \multicolumn{2}{c}{k=10}                             & \multicolumn{2}{c}{k=20}                             & \multicolumn{2}{c}{k=10}                             & \multicolumn{2}{c}{k=20}                                                     & \multicolumn{2}{c}{k=10}                             & \multicolumn{2}{c}{k=20}                             & \multicolumn{2}{c}{k=10}                             & \multicolumn{2}{c}{k=20}                             & \multicolumn{2}{c}{k=10}                             & \multicolumn{2}{c}{k=20}                             & \multicolumn{2}{c}{k=10}                             & \multicolumn{2}{c}{k=20}                             \\
\multicolumn{1}{c}{\multirow{-4}{*}{Models}} & \multicolumn{1}{c}{$C_V$} & \multicolumn{1}{c}{$TD$} & \multicolumn{1}{c}{$C_V$} & \multicolumn{1}{c}{$TD$} & \multicolumn{1}{c}{$C_V$} & \multicolumn{1}{c}{$TD$} & \multicolumn{1}{c}{$C_V$} & \multicolumn{1}{c}{$TD$}                         & \multicolumn{1}{c}{$C_V$} & \multicolumn{1}{c}{$TD$} & \multicolumn{1}{c}{$C_V$} & \multicolumn{1}{c}{$TD$} & \multicolumn{1}{c}{$C_V$} & \multicolumn{1}{c}{$TD$} & \multicolumn{1}{c}{$C_V$} & \multicolumn{1}{c}{$TD$} & \multicolumn{1}{c}{$C_V$} & \multicolumn{1}{c}{$TD$} & \multicolumn{1}{c}{$C_V$} & \multicolumn{1}{c}{$TD$} & \multicolumn{1}{c}{$C_V$} & \multicolumn{1}{c}{$TD$} & \multicolumn{1}{c}{$C_V$} & \multicolumn{1}{c}{$TD$} \\ \cmidrule(lr){2-25}
LDA                                         & 0.425                     & 0.696                    & 0.429                     & 0.564                    & 0.424                     & 0.588                    & 0.381                     & \multicolumn{1}{r}{{\color[HTML]{3B3B3B} 0.508}} & 0.439                     & 0.420                    & 0.444                     & 0.372                    & 0.461                     & 0.472                    & 0.457                     & 0.318                    & 0.394                     & 0.420                    & 0.398                     & 0.358                    & 0.351                     & 0.680                    & 0.391                     & 0.662                    \\
ProdLDA                                     & 0.410                     & 0.816                    & 0.417                     & 0.834                    & 0.422                     & 0.900                    & 0.390                     & 0.878                                            & 0.437                     & 0.796                    & 0.473                     & 0.616                    & 0.472                     & 0.720                    & 0.403                     & 0.662                    & 0.437                     & 0.772                    & 0.453                     & 0.834                    & 0.405                     & 0.828                    & 0.421                     & 0.708                    \\
ETM                                         & 0.469                     & 0.808                    & 0.408                     & 0.498                    & 0.422                     & 0.784                    & 0.406                     & 0.560                                            & 0.439                     & 0.516                    & 0.426                     & 0.304                    & 0.445                     & 0.492                    & 0.450                     & 0.314                    & 0.359                     & 0.688                    & 0.412                     & 0.518                    & 0.434                     & 0.632                    & 0.407                     & 0.336                    \\
CTM                                         & 0.476                     & 0.804                    & 0.431                     & 0.832                    & 0.407                     & 0.852                    & 0.422                     & 0.830                                            & 0.466                     & 0.792                    & 0.470                     & 0.694                    & 0.422                     & 0.724                    & 0.480                     & 0.594                    & 0.398                     & 0.768                    & 0.441                     & 0.746                    & 0.471                     & 0.848                    & 0.476                     & 0.732                    \\
ECRTM                                       & 0.391                     & 0.636                    & 0.427                     & 0.556                    & 0.410                     & 0.628                    & 0.420                     & 0.524                                            & 0.459                     & 0.632                    & 0.410                     & 0.860                    & 0.411                     & 0.596                    & 0.457                     & 0.798                    & 0.392                     & 0.728                    & 0.473                     & 0.472                    & 0.499                     & 0.829                    & 0.493                     & \textbf{0.821}           \\
DeTime                                      & 0.417                     & 0.808                    & 0.411                     & 0.844                    & 0.402                     & 0.900                    & 0.396                     & 0.874                                            & 0.355                     & 0.672                    & 0.341                     & 0.652                    & 0.380                     & 0.714                    & 0.345                     & 0.648                    & 0.371                     & 0.716                    & 0.374                     & 0.784                    & 0.378                     & 0.628                    & 0.382                     & 0.562                    \\
Meta-CETM                                   & 0.396                     & 0.831                    & 0.391                     & 0.891                    & 0.409                     & 0.873                    & 0.403                     & 0.899                                            & 0.493                     & 0.845                    & 0.530                     & 0.748                    & 0.426                     & 0.679                    & 0.417                     & 0.872                    & 0.406                     & 0.791                    & 0.437                     & 0.761                    & 0.452                     & 0.879                    & 0.423                     & 0.792                    \\
Fastopic                                    & 0.406                     & 0.829                    & 0.424                     & 0.905                    & 0.389                     & 0.881                    & 0.418                     & 0.900                                            & 0.517                     & 0.811                    & 0.490                     & \textbf{0.948}           & 0.413                     & 0.709                    & 0.414                     & 0.900                    & 0.440                     & 0.811                    & 0.454                     & 0.778                    & 0.464                     & 0.814                    & 0.485                     & 0.692                    \\ \midrule  
DALTA                                       & \textbf{0.493}            & \textbf{0.836}           & \textbf{0.451}            & \textbf{0.924}           & \textbf{0.431}            & \textbf{0.908}           & \textbf{0.451}            & \textbf{0.918}                                   & \textbf{0.582}            & \textbf{0.892}           & \textbf{0.571}            & 0.800                    & \textbf{0.484}            & \textbf{0.732}           & \textbf{0.483}            & \textbf{0.932}           & \textbf{0.448}            & \textbf{0.852}           & \textbf{0.516}            & \textbf{0.808}           & \textbf{0.503}            & \textbf{0.900}           & \textbf{0.505}            & 0.800   \\ \bottomrule                  
\end{tabular}}
\caption{Topic coherences ($C_V$) and diversity ($TD$) scores of topic words. $k$ denotes the number of topics. The best result in each case is shown in \textbf{bold}.}
\label{tab:coherence_ag_1000samples}
\end{table*}

%% file: tables/classification_ag_1000samples.tex
\begin{table*}[!ht]
\resizebox{1.0\linewidth}{!}{%
\begin{tabular}{@{}lllll|llll|llll|llll|llll|llll@{}}
\toprule
\multicolumn{1}{l}{\multirow{4}{*}{Models}} & \multicolumn{8}{c}{Newsgroup}                                                                                                                                                                                               & \multicolumn{8}{c}{Drug Review}                                                                                                                                                                           & \multicolumn{4}{c}{\multirow{2}{*}{Yelp}}                                                              & \multicolumn{4}{c}{\multirow{2}{*}{\begin{tabular}[c]{@{}c@{}}SMS Spam \\ Collection\end{tabular}}}     \\ 
\multicolumn{1}{c}{}                       & \multicolumn{4}{c}{Science}                                                                                    & \multicolumn{4}{c}{Religion}                                                                        & \multicolumn{4}{c}{Norethindrone}                                                                   & \multicolumn{4}{c}{Norgestimate}                                                                    & \multicolumn{4}{c}{}                                                                                   & \multicolumn{4}{c}{}                                                                                \\ \cmidrule(lr){2-25}
\multicolumn{1}{c}{}                       & \multicolumn{2}{c}{k=10}                                    & \multicolumn{2}{c}{k=20}                         & \multicolumn{2}{c}{k=10}                         & \multicolumn{2}{c}{k=20}                         & \multicolumn{2}{c}{k=10}                         & \multicolumn{2}{c}{k=20}                         & \multicolumn{2}{c}{k=10}                         & \multicolumn{2}{c}{k=20}                         & \multicolumn{2}{c}{k=10}                         & \multicolumn{2}{c}{k=20}                            & \multicolumn{2}{c}{k=10}                         & \multicolumn{2}{c}{k=20}                         \\ 
\multicolumn{1}{c}{}                       & \multicolumn{1}{c}{SVC}            & \multicolumn{1}{c}{LR} & \multicolumn{1}{c}{SVC} & \multicolumn{1}{c}{LR} & \multicolumn{1}{c}{SVC} & \multicolumn{1}{c}{LR} & \multicolumn{1}{c}{SVC} & \multicolumn{1}{c}{LR} & \multicolumn{1}{c}{SVC} & \multicolumn{1}{c}{LR} & \multicolumn{1}{c}{SVC} & \multicolumn{1}{c}{LR} & \multicolumn{1}{c}{SVC} & \multicolumn{1}{c}{LR} & \multicolumn{1}{c}{SVC} & \multicolumn{1}{c}{LR} & \multicolumn{1}{c}{SVC} & \multicolumn{1}{c}{LR} & \multicolumn{1}{c}{SVC} & \multicolumn{1}{c}{LR}    & \multicolumn{1}{c}{SVC} & \multicolumn{1}{c}{LR} & \multicolumn{1}{c}{SVC} & \multicolumn{1}{c}{LR} \\ \cmidrule(lr){2-25}
LDA                                        & 0.575                              & 0.588                  & 0.540                   & 0.547                  & 0.505                   & 0.513                  & 0.529                   & 0.522                  & 0.564                   & 0.564                  & 0.562                   & 0.578                  & 0.571                   & 0.557                  & 0.575                   & 0.600                  & 0.686                   & 0.686                  & 0.686                   & 0.686                     & 0.864                   & 0.872                  & 0.890                   & 0.897                  \\
ProdLDA                                    & 0.569                              & 0.635                  & 0.650                   & \textbf{0.716}         & 0.496                   & 0.549                  & 0.473                   & 0.505                  & 0.587                   & 0.569                  & 0.570                   & 0.540                  & 0.599                   & 0.620                  & 0.615                   & 0.628                  & 0.667                   & 0.656                  & 0.686                   & 0.686                     & 0.837                   & 0.960                  & 0.864                   & 0.864                  \\
ETM                                        & 0.297                              & 0.272                  & 0.246                   & 0.263                  & 0.404                   & 0.394                  & 0.404                   & 0.378                  & 0.451                   & 0.452                  & 0.433                   & 0.442                  & 0.493                   & 0.496                  & 0.495                   & 0.478                  & 0.686                   & 0.686                  & 0.686                   & 0.686                     & 0.864                   & 0.864                  & 0.864                   & 0.864                  \\
CTM                                        & 0.674                              & 0.651                  & 0.665                   & 0.702                  & 0.544                   & 0.557                  & 0.498                   & 0.516                  & 0.566                   & 0.552                  & 0.586                   & 0.576                  & 0.562                   & 0.563                  & 0.627                   & 0.628                  & 0.686                   & 0.686                  & 0.686                   & 0.686                     & 0.886                   & 0.987                  & 0.864                   & 0.864                  \\
ECRTM                                      & 0.534                              & 0.566                  & 0.591                   & 0.625                  & 0.521                   & 0.541                  & 0.516                   & 0.507                  & 0.584                   & 0.596                  & 0.592                   & 0.552                  & 0.542                   & 0.542                  & 0.630                   & 0.613                  & 0.686                   & 0.685                  & 0.686                   & 0.686                     & 0.881                   & 0.883                  & 0.870                   & 0.891                  \\
DeTime                                     & 0.254                              & 0.254                  & 0.254                   & 0.254                  & 0.411                   & 0.411                  & 0.410                   & 0.410                  & 0.464                   & 0.464                  & 0.464                   & 0.464                  & 0.503                   & 0.503                  & 0.503                   & 0.503                  & 0.686                   & 0.686                  & 0.686                   & 0.686                     & 0.864                   & 0.864                  & 0.864                   & 0.864                  \\
Meta-CETM                                  & 0.681                              & 0.729                  & 0.641                   & 0.682                  & 0.506                   & 0.492                  & 0.489                   & 0.515                  & 0.558                   & 0.519                  & 0.543                   & 0.561                  & 0.601                   & 0.591                  & 0.635                   & 0.636                  & 0.649                   & 0.686                  & 0.686       & 0.686            & 0.871                   & 0.872                  & 0.853                   & 0.881                  \\
Fastopic                                   & 0.678                              & 0.702                  & 0.667                   & 0.860                  & 0.510                   & 0.522                  & 0.519                   & 0.530                  & 0.564                   & 0.582                  & 0.572                   & 0.584                  & 0.616                   & 0.601                  & 0.625                   & 0.636                  & 0.685                   & 0.686                  & \textbf{0.717}          & \textbf{0.717}            & 0.869                   & 0.870                  & 0.869                   & 0.868                              \\ \midrule
DALTA                                      & \multicolumn{1}{r}{\textbf{0.698}} & \textbf{0.758}         & \textbf{0.685}          & 0.707                  & \textbf{0.529}          & \textbf{0.549}         & \textbf{0.528}          & \textbf{0.549}         & \textbf{0.598}          & \textbf{0.600}         & \textbf{0.598}          & \textbf{0.600}         & \textbf{0.646}          & \textbf{0.641}         & \textbf{0.646}          & \textbf{0.641}         & \textbf{0.686}          & \textbf{0.686}         & 0.686                   & 0.684 & \textbf{0.975}          & \textbf{0.978}         & \textbf{0.975}          & \textbf{0.978} \\ \bottomrule      
\end{tabular}%
}
\caption{Text classification accuracy over 5-fold cross-validation. The best results in each case are shown in \textbf{bold}.}
\label{tab:classification_ag_1000samples}
\end{table*}

%% file: tables/clustering_ag_1000samples.tex
\begin{table*}[!ht]
\resizebox{1.0\linewidth}{!}{%
\begin{tabular}{@{}lllll|llll|llll|llll|llll|llll@{}}
\toprule
\multicolumn{1}{c}{\multirow{4}{*}{Models}} & \multicolumn{8}{c}{Newsgroup}                                                                                                                                                                                             & \multicolumn{8}{c}{Drug Review}                                                                                                                                                                                           & \multicolumn{4}{c}{\multirow{2}{*}{Yelp}}                                                                   & \multicolumn{4}{c}{\multirow{2}{*}{\begin{tabular}[c]{@{}c@{}}SMS Spam \\ Collection\end{tabular}}}         \\
\multicolumn{1}{c}{}                        & \multicolumn{4}{c}{Science}                                                                                 & \multicolumn{4}{c}{Religion}                                                                                & \multicolumn{4}{c}{Norethindrone}                                                                           & \multicolumn{4}{c}{Norgestimate}                                                                            & \multicolumn{4}{c}{}                                                                                        & \multicolumn{4}{c}{}                                                                                        \\ \cmidrule(lr){2-25}
\multicolumn{1}{c}{}                        & \multicolumn{2}{c}{k=10}                             & \multicolumn{2}{c}{k=20}                             & \multicolumn{2}{c}{k=10}                             & \multicolumn{2}{c}{k=20}                             & \multicolumn{2}{c}{k=10}                             & \multicolumn{2}{c}{k=20}                             & \multicolumn{2}{c}{k=10}                             & \multicolumn{2}{c}{k=20}                             & \multicolumn{2}{c}{k=10}                             & \multicolumn{2}{c}{k=20}                             & \multicolumn{2}{c}{k=10}                             & \multicolumn{2}{c}{k=20}                             \\
\multicolumn{1}{c}{}                        & \multicolumn{1}{c}{Purity} & \multicolumn{1}{c}{NMI} & \multicolumn{1}{c}{Purity} & \multicolumn{1}{c}{NMI} & \multicolumn{1}{c}{Purity} & \multicolumn{1}{c}{NMI} & \multicolumn{1}{c}{Purity} & \multicolumn{1}{c}{NMI} & \multicolumn{1}{c}{Purity} & \multicolumn{1}{c}{NMI} & \multicolumn{1}{c}{Purity} & \multicolumn{1}{c}{NMI} & \multicolumn{1}{c}{Purity} & \multicolumn{1}{c}{NMI} & \multicolumn{1}{c}{Purity} & \multicolumn{1}{c}{NMI} & \multicolumn{1}{c}{Purity} & \multicolumn{1}{c}{NMI} & \multicolumn{1}{c}{Purity} & \multicolumn{1}{c}{NMI} & \multicolumn{1}{c}{Purity} & \multicolumn{1}{c}{NMI} & \multicolumn{1}{c}{Purity} & \multicolumn{1}{c}{NMI} \\ \cmidrule(lr){2-25}
LDA                                         & 0.41                       & 0.152                   & 0.5                        & 0.099                   & 0.406                      & 0.037                   & 0.4                        & 0.032                   & 0.511                      & 0.045                   & 0.471                      & 0.043                   & 0.526                      & 0.069                   & 0.577                      & 0.032                   & 0.686                      & 0.006                   & 0.686                      & 0.017                   & 0.864                      & 0.039                   & 0.884                      & 0.052                   \\
ProdLDA                                     & 0.411                      & 0.202                   & 0.481                      & 0.253          & 0.462                      & 0.049                   & 0.477                      & 0.051                   & 0.541                      & 0.035                   & 0.466                      & 0.018                   & 0.503                      & 0                       & 0.554                      & 0.039                   & 0.686                      & 0.006                   & 0.686                      & 0.007                   & 0.864                      & 0.003                   & 0.864                      & 0                       \\
ETM                                         & 0.308                      & 0.01                    & 0.311                      & 0.015                   & 0.413                      & 0.006                   & 0.429                      & 0.012                   & 0.485                      & 0.006                   & 0.475                      & 0.009                   & 0.507                      & 0.004                   & 0.508                      & 0.008                   & 0.686                      & 0.007                   & 0.686                      & 0.013                   & 0.864                      & 0.003                   & 0.864                      & 0.003                   \\
CTM                                         & 0.495                      & 0.149                   & 0.73                       & \textbf{0.371}          & 0.499                      & \textbf{0.067}          & 0.495                      & 0.058                   & 0.6                        & 0.057                   & 0.413                      & 0.028                   & 0.532                      & 0.07                    & 0.51                       & 0.063                   & 0.686                      & 0.014                   & 0.69                       & 0.025                   & 0.964                      & 0.233                   & 0.864                      & 0                       \\
ECRTM                                       & 0.458                      & 0.126                   & 0.527                      & 0.2                     & 0.408                      & \textbf{0.067}          & 0.428                      & 0.043                   & 0.474                      & 0.04                    & 0.557                      & 0.037                   & 0.584                      & 0.052                   & 0.491                      & 0.045                   & 0.686                      & 0.011                   & 0.686                      & 0.016                   & 0.88                       & 0.095                   & 0.864                      & 0.032                   \\
DeTime                                      & 0.254                      & 0                       & 0.254                      & 0                       & 0.411                      & 0                       & 0.411                      & 0                       & 0.464                      & 0                       & 0.464                      & 0                       & 0.503                      & 0                       & 0.403                      & 0                       & 0.686                      & 0                       & 0.686                      & 0                       & 0.864                      & 0                       & 0.864                      & 0                       \\
Meta-CETM                                   & 0.415                      & 0.218                   & 0.553                      & 0.245                   & 0.441                      & 0.055                   & 0.464                      & 0.031                   & 0.444                      & 0.057                   & 0.445                      & 0.076                   & 0.571                      & 0.041                   & 0.566                      & 0.07                    & 0.686                      & 0.022                   & 0.686                      & \textbf{0.045}          & 0.928                      & 0.168                   & 0.937                      & 0.166                   \\
Fastopic                                    & 0.437                      & 0.211                   & 0.539                      & 0.243                   & 0.431                      & 0.065                   & 0.454                      & 0.031                   & 0.494                      & \textbf{0.057}          & 0.445                      & 0.076                   & 0.582                      & 0.05                    & 0.536                      & 0.07                    & 0.686                      & 0.032                   & \textbf{0.699}             & \textbf{0.035}          & 0.928                      & 0.168                   & 0.937                      & 0.166                   \\ \midrule
DALTA                                       & \textbf{0.513}             & \textbf{0.225}          & \textbf{0.579}             & 0.274                   & \textbf{0.5}               & 0.059          & \textbf{0.5}               & \textbf{0.059}          & \textbf{0.552}             & 0.032          & \textbf{0.552}             & \textbf{0.052}          & \textbf{0.604}             & \textbf{0.071}          & \textbf{0.604}             & \textbf{0.071}          & \textbf{0.686}             & \textbf{0.007}          & \textbf{0.696}             & 0.009                   & \textbf{0.978}             & \textbf{0.287}          & \textbf{0.978}             & \textbf{0.287}  \\ \bottomrule        
\end{tabular}      
}
\caption{Document clustering results of Purity and NMI. The best results in each case are shown in \textbf{bold}.}
\label{tab:clustering_ag_1000samples}
\end{table*}

%% file: tables/ablation_study.tex
\begin{table*}[!ht]
\centering
\resizebox{0.8\linewidth}{!}{%
\begin{tabular}{lllllllllllll}
\toprule
\multirow{4}{*}{$\mathcal{L}_{rec_T}$} & \multirow{4}{*}{$\mathcal{L}_{rec_S}$} & \multirow{4}{*}{$\mathcal{L}_{adv}$} & \multirow{4}{*}{$\mathcal{L}_{cons}$} & \multirow{4}{*}{$\mathcal{L}_{KL}$} & \multicolumn{8}{c}{NG-Science}                                                                                                                                                                                               \\
                                       &                                        &                                      &                                       &                                     & \multicolumn{4}{c}{Topic Quality}                                                                           & \multicolumn{4}{c}{Classification}                                                                             \\ \cmidrule(lr){6-13}
                                       &                                        &                                      &                                       &                                     & \multicolumn{2}{c}{k=10}                             & \multicolumn{2}{c}{k=20}                             & \multicolumn{2}{c}{k=10}                                    & \multicolumn{2}{c}{k=20}                         \\
                                       &                                        &                                      &                                       &                                     & \multicolumn{1}{c}{$C_V$} & \multicolumn{1}{c}{$TD$} & \multicolumn{1}{c}{$C_V$} & \multicolumn{1}{c}{$TD$} & \multicolumn{1}{c}{SVC}            & \multicolumn{1}{c}{LR} & \multicolumn{1}{c}{SVC} & \multicolumn{1}{c}{LR} \\ \midrule
\Checkmark                           & \Xmark                               & \Xmark                             & \Xmark                              & \Xmark                            & 0.458                     & 0.848                    & 0.391                     & 0.868                    & 0.614                              & 0.704                  & 0.671                   & 0.684                  \\
\Checkmark                           & \Checkmark                           & \Xmark                             & \Xmark                              & \Xmark                            & 0.365                     & 0.808                    & 0.305                     & 0.824                    & 0.658                              & 0.725                  & 0.685                   & 0.707                  \\
\Checkmark                           & \Xmark                               & \Checkmark                         & \Xmark                              & \Xmark                            & 0.427                     & 0.860                    & 0.394                     & 0.864                    & 0.673                              & 0.693                  & 0.656                   & 0.702                  \\
\Checkmark                           & \Xmark                               & \Xmark                             & \Checkmark                          & \Xmark                            & 0.350                     & 0.844                    & 0.340                     & 0.844                    & 0.676                              & 0.709                  & 0.635                   & \textbf{0.713}         \\
\Checkmark                           & \Xmark                               & \Xmark                             & \Xmark                              & \Checkmark                        & 0.385                     & 0.856                    & 0.393                     & 0.864                    & 0.683                              & 0.705                  & 0.663                   & 0.675                  \\
\Checkmark                           & \Xmark                               & \Xmark                             & \Checkmark                          & \Checkmark                        & 0.354                     & 0.844                    & 0.406                     & 0.836                    & 0.671                              & 0.723                  & 0.680                   & 0.689                  \\
\Checkmark                           & \Xmark                               & \Checkmark                         & \Checkmark                          & \Checkmark                        & 0.383                     & \textbf{0.868}           & 0.373                     & 0.836                    & 0.669                              & 0.718                  & 0.649                   & 0.667                  \\
\Checkmark                           & \Checkmark                           & \Checkmark                         & \Checkmark                          & \Checkmark                        & \textbf{0.493}            & 0.836                    & \textbf{0.451}            & \textbf{0.924}           & \multicolumn{1}{r}{\textbf{0.698}} & \textbf{0.758}         & \textbf{0.685}          & 0.707   \\ \bottomrule              
\end{tabular}}
\caption{Ablation study. “\Checkmark” means we use the
corresponding loss term and "\Xmark" is otherwise.}
\label{tab:ablation}
\end{table*}

%% file: tables/qualitative_result.tex
\begin{table*}[ht]
\centering
\resizebox{0.8\linewidth}{!}{%
\begin{tabular}{lll}
\toprule
\textbf{Model} & \textbf{Topic 1} & \textbf{Topic 2} \\
\midrule
LDA        & good, place, great, food, come         & like, time, place, come, want \\
ProdLDA    & good, food, like, menu, order          & schnitzel, alligator, shell, fry, foodie \\
ETM        & burger, pizza, sandwich, salad, meal   & like, know, come, say, time \\
CTM        & chicken, soup, beet, salad, mandarin   & good, order, come, say, like \\
ECRTM      & order, food, table, burger, salad      & taco, pho, boba, chop, gelato \\
DeTime     & flavor, waffle, soup, decor, slaw      & great, love, order, time, try \\
Meta-CETM  & burger, fries, shake, cheese, ketchup  & spa, massage, facial, candle, relaxing \\
Fastopic   & burger, pork, rice, sushi, sandwich    & server, minutes, wait, thanks, manager \\
\textbf{DALTA} & \textbf{taco, burrito, salsa, mexican, lunch} & \textbf{kitchen, party, room, night, drink} \\
\bottomrule
\end{tabular}}
\caption{Representative 5-word topics generated by each model on the Yelp Reviews dataset.}
\label{tab:qual-topic}
\end{table*}

%% file: sections/conclusion.tex
\section{Conclusion}
\label{sec:conclusion}
In this paper, we address the challenge of cross-domain adaptation in low-resource topic modeling by introducing a theoretical generalization bound that quantifies the conditions for effective knowledge transfer. Based on this insight, we propose DALTA, a model that disentangles domain-invariant and domain-specific components to improve topic adaptation. Given a high-resource source corpus and a low-resource target corpus, DALTA learns to align topics while preserving domain-specific information. The empirical results demonstrate that DALTA consistently outperforms state-of-the-art methods, highlighting the effectiveness of our framework in low-resource settings.

%% file: sections/limitation.tex
\section*{Limitations}
While DALTA provides a strong theoretical foundation for cross-domain adaptation, it does not offer a practical method for selecting the most suitable source domain. However, we find that its internal training signals can be repurposed for this purpose. In Appendix \ref{sec:casestudy}, we present case studies demonstrating a simple alignment score that effectively identifies optimal source domains without requiring full model convergence. Although our theoretical bound helps determine when knowledge transfer is beneficial, we do not empirically investigate how different source domains influence adaptation performance. This raises an open question on how to systematically identify the best source domain for a given low-resource target domain. Additionally, DALTA’s effectiveness depends on the degree of topic structure alignment between domains, which may not always be known beforehand. If the source and target domains differ significantly in topic distributions, adaptation may lead to misalignment or reduced model performance. Addressing these challenges requires further research into automated source domain selection strategies that can optimize adaptation across diverse real-world settings.

%% file: sections/acknowledgements.tex
\section*{Acknowledgements}
This material is based upon work supported by the National Science Foundation IIS 16-19302 and IIS 16-33755, Zhejiang University ZJU Research 083650, IBM-Illinois Center for Cognitive Computing Systems Research (C3SR) and IBM-Illinois Discovery Accelerator Institute (IIDAI), grants from eBay and Microsoft Azure, UIUC OVCR CCIL Planning Grant 434S34, UIUC CSBS Small Grant 434C8U, and UIUC New Frontiers Initiative. Any opinions, findings, conclusions, or recommendations expressed in this publication are those of the author(s) and do not necessarily reflect the views of the funding agencies.

%% file: sections/appendix.tex
\section{Proof of generalization bound}

\begin{lemma} \cite{ben2010theory}
\label{lemma1}
Let $\mathcal{X} \subseteq \mathbb{R}^d$ be the instance space, and let $\mathcal{D}_S$ and $\mathcal{D}_T$ represent the source and target data distributions over $\mathcal{X}$. $f_S$ and $f_T$ are the optimal labeling function for the source and target domains, respectively. For a hypothesis $h \in \mathcal{H}$, the target domain error $\epsilon_T(h)$ is bounded as:
\begin{align*}
\epsilon_T(h) &\leq \epsilon_S(h) + d_\mathcal{H}(\mathcal{D}_S(X), \mathcal{D}_T(X)) \\ &+ \min \left\{ \mathbb{E}_{S} \big[ |f_S - f_T| \big], \mathbb{E}_{T} \big[ |f_S - f_T| \big] \right\}.
\end{align*}

\end{lemma}







\begin{lemma}[Reconstruction Guarantee for Bounded Instance Spaces \cite{mbacke2024statistical}]
\label{lemma2}
Let \( \mathcal{X} \) be the instance space with diameter \( \Delta < \infty \), and let \( \mu \in \mathcal{M}_+^1(\mathcal{X}) \) denote the data-generating distribution. Consider \( \mathcal{Z} \) as the latent space with a prior distribution \( p(z) \in \mathcal{M}_+^1(\mathcal{Z}) \), and let \( \theta \) represent the parameters of the reconstruction function. For any posterior distribution \( q_\phi(z|\mathbf{x}) \), regularization parameter \( \lambda > 0 \), and confidence level \( \delta \in (0, 1) \), the following inequality holds with probability at least \( 1 - \delta \) over a random sample \( S \sim \mu^{\otimes n} \):

\begin{align*}
\epsilon(h) \leq \hat{\epsilon}(h) + &\frac{1}{\lambda} \text{KL}(q \| p) + K_\phi K_\theta \Delta \\
&+ \frac{1}{\lambda} \log \frac{1}{\delta} + \frac{\lambda \Delta^2}{8n},
\end{align*}

where

$\epsilon(h): \mathbb{E}_{\mathbf{x} \sim \mu} \mathbb{E}_{z \sim q_\phi(z|\mathbf{x})} \ell^\theta_{\text{rec}}(z, \mathbf{x})$, the expected reconstruction loss over the true data distribution,

$\hat{\epsilon}(h): \frac{1}{n} \sum_{i=1}^n \mathbb{E}_{z \sim q_\phi(z|\mathbf{x}_i)} \ell^\theta_{\text{rec}}(z, \mathbf{x}_i)$, the empirical reconstruction loss over the available data,

$\text{KL}(q\| p): \sum_{i=1}^n \text{KL}(q_\phi(z|\mathbf{x}_i) \| p(z))$, the KL divergence between the posterior and the prior.

\end{lemma}

\begin{lemma}
\label{lamma3}


Let  $h \in \mathcal{H}$  be a hypothesis from a hypothesis class  $\mathcal{H}$ , where  $h: \mathcal{Z} \to [0,1]^{|\mathbb{V}|}$  maps from a latent semantic space  $\mathcal{Z}$  to a probability distribution over a vocabulary  $\mathbb{V}$. \( f_S \) and \( f_T \) are the optimal functions mapping latent representations to reconstructed outputs for the source and target domains, respectively. Then, for every \( h \in \mathcal{H} \) and for any \( \delta > 0 \), with probability at least \( 1 - \delta \) over the choice of the source and target samples of sizes \( n_S \) and \( n_T \), the following inequality holds:

\resizebox{\columnwidth}{!}{%
\begin{minipage}{\columnwidth} 
\begin{align*}
\epsilon_T(h) &\leq \hat{\epsilon}_S(h) + d_\mathcal{H}(\mathcal{D}_S(Z), \mathcal{D}_T(Z)) + \\ & \min \left\{ \mathbb{E}_{S} \big[| f_S - f_T | \big], \, \mathbb{E}_{T} \big[ | f_S - f_T | \big] \right\}+ \\ & \frac{1}{\lambda} \text{KL}(q_S \| p) + K_\phi K_{\theta_S} \Delta +
 \frac{1}{\lambda} \log \frac{1}{\delta} + \frac{\lambda \Delta^2}{8n_S}
\end{align*}
\end{minipage}%
}
\end{lemma}

\begin{proof}
Following \textbf{Lemma} \ref{lemma1}, for the hypothesis $h$, the target domain reconstruction error is bounded by:
\begin{align*}
\epsilon_T(h) &\leq \epsilon_S(h) + d_\mathcal{H}(\mathcal{D}_S(Z), \mathcal{D}_T(Z)) \\ &+ \min \left\{ \mathbb{E}_{S} \big[| f_S - f_T | \big], \, \mathbb{E}_{T} \big[ | f_S - f_T | \big] \right\},
\end{align*} 

Now, invoking the upper bound in Lemma \ref{lemma2}, we have with probability at least $1 - \delta$, 
\resizebox{\columnwidth}{!}{%
\begin{minipage}{\columnwidth}
\begin{align*}
\epsilon_T(h) &\leq \hat{\epsilon}_S(h) + d_\mathcal{H}(\mathcal{D}_S(Z), \mathcal{D}_T(Z)) \\ &+ \min \left\{ \mathbb{E}_{S} \big[| f_S - f_T | \big], \, \mathbb{E}_{T} \big[ | f_S - f_T | \big] \right\} \\ &+ \frac{1}{\lambda} \text{KL}(q_S \| p) + K_\phi K_{\theta_S} \Delta +
 \frac{1}{\lambda} \log \frac{1}{\delta} + \frac{\lambda \Delta^2}{8n_S}
\end{align*}
\end{minipage}%
}
\end{proof}

\setcounter{theorem}{0}
\begin{theorem} [Generalization bound]

Let $h \in \mathcal{H}$ be a hypothesis from a hypothesis class $\mathcal{H}$, where $h: \mathcal{Z} \to [0,1]^{|\mathcal{V}|}$ maps from a latent semantic space $\mathcal{Z}$ to a probability distribution over a vocabulary $\mathcal{V}$. Let $f_S$ and $f_T$ be the optimal functions mapping latent representations to reconstructed outputs for the source and target domains, respectively. Define $p_S = \frac{n_S}{n_S + n_T}$ as the proportion of source samples and $p_T = \frac{n_T}{n_S + n_T}$ as the proportion of target samples. Then, for every $h \in \mathcal{H}$ and for any $\delta > 0$, with probability at least $1 - \delta$ over the choice of the source and target samples of sizes $n_S$ and $n_T$, the target domain error is bounded by:

\resizebox{\columnwidth}{!}{%
\begin{minipage}{\columnwidth}
    \begin{align*}
    \epsilon_T(h) &\leq p_T\cdot\hat{\epsilon}_T(h) + p_S\cdot\hat{\epsilon}_S(h) \\&+ \frac{1}{\lambda}\text{KL}(q \| p)
     +p_S\cdot(d_\mathcal{H}(\mathcal{D}_S(Z), \mathcal{D}_T(Z)) \\&+ p_S\cdot\min \left\{ \mathbb{E}_{S} \big[| f_S - f_T | \big], \, \mathbb{E}_{T} \big[ | f_S - f_T | \big] \right\}  \\&+
     \mathcal{O}\Big(K_\phi K_\theta \Delta + \frac{1}{\lambda} \log \frac{1}{\delta} + \frac{\lambda \Delta^2}{n_S + n_T}\Big)
    \end{align*}
\end{minipage}%
}
\end{theorem}
\begin{proof}
Having Lamma \ref{lemma2} and \ref{lamma3}, we can use a union bound to combine
them with coefficients $\frac{n_T}{n_S + n_T}$ and $\frac{n_S}{n_S + n_T}$ respectively as follows:
\resizebox{\columnwidth}{!}{%
\begin{minipage}{\columnwidth}
\begin{align*}
\epsilon_T(h) &\leq
\frac{n_T}{n_S + n_T}\Big(
\hat{\epsilon}_T(h) + \frac{1}{\lambda} \text{KL}(q_T \| p) 
\\ &  + K_\phi K_{\theta_T} \Delta +  \frac{1}{\lambda} \log \frac{1}{\delta} + \frac{\lambda \Delta^2}{8n_T}\Big) \\ &+ \frac{n_S}{n_S + n_T}\Big(
\hat{\epsilon}_S(h) + d_\mathcal{H}(\mathcal{D}_S(Z), \mathcal{D}_T(Z)) \\ & + \min \left\{ \mathbb{E}_{S} \big[| f_S - f_T | \big], \, \mathbb{E}_{T} \big[ | f_S - f_T | \big] \right\} \\ &+ \frac{1}{\lambda} \text{KL}(q_S \| p) + K_\phi K_{\theta_S} \Delta +
 \frac{1}{\lambda} \log \frac{1}{\delta} + \frac{\lambda \Delta^2}{8n_S} \Big) \\
 &= \frac{n_T}{n_S + n_T}\hat{\epsilon}_T(h) + \frac{n_S}{n_S + n_T}\hat{\epsilon}_S(h) \\&+ \frac{1}{\lambda}\text{KL}(q \| p) +
 \frac{n_S}{n_S + n_T} \Big( d_\mathcal{H}(\mathcal{D}_S(Z), \mathcal{D}_T(Z)) \\&+ \min \left\{ \mathbb{E}_{S} \big[| f_S - f_T | \big], \, \mathbb{E}_{T} \big[ | f_S - f_T | \big] \right\} \Big) \\&+ K_\phi K_\theta \Delta + \frac{1}{\lambda} \log \frac{1}{\delta} + \frac{\lambda \Delta^2}{4(n_S + n_T)}\\
 &\leq p_T\cdot\hat{\epsilon}_T(h) + p_S\cdot\hat{\epsilon}_S(h) \\&+ \frac{1}{\lambda}\text{KL}(q \| p) +
 p_S\cdot d_\mathcal{H}(\mathcal{D}_S(Z), \mathcal{D}_T(Z)) \\&+ p_S\cdot\min \left\{ \mathbb{E}_{S} \big[| f_S - f_T | \big], \, \mathbb{E}_{T} \big[ | f_S - f_T | \big] \right\} \\&+
 \mathcal{O}\Big(K_\phi K_\theta \Delta + \frac{1}{\lambda} \log \frac{1}{\delta} + \frac{\lambda \Delta^2}{n_S + n_T}\Big)
\end{align*}
\end{minipage}%
}
\end{proof}

\setcounter{proposition}{0}
\begin{proposition}
Let \( q_\phi: \mathcal{X} \to \mathcal{Z} \) be a shared encoder mapping documents from the source domain \( \mathcal{X}_S \) and target domain \( \mathcal{X}_T \) into a common latent space \( \mathcal{Z} \). The \(\mathcal{H}\)-divergence between the source and target latent distributions is given by:
\[
d_{\mathcal{H}}(q_\phi(\mathcal{X}_S), q_\phi(\mathcal{X}_T)) = 2 \left(1 - 2 \epsilon_{C^*} \right),
\]
where \( \epsilon_{C^*} \) is the classification error of the optimal domain classifier \( C^* \).
\end{proposition}
\begin{proof}
Consider the optimal domain classifier \( C^* \), which distinguishes between the source and target latent representations. The optimal decision function is given by:
\[
C^*(z) = \frac{p_S(z)}{p_S(z) + p_T(z)},
\]
where \( p_S(z) \) and \( p_T(z) \) denote the probability density functions of the source and target latent representations, respectively.

The classification error of \( C^* \) can be expressed as:
\[
\epsilon_{C^*} = \frac{1}{2} \int \min(p_S(z), p_T(z)) \, dz.
\]

Using the identity \( \min(a, b) = \frac{a + b - |a - b|}{2} \), we obtain:
\[
\int \min(p_S(z), p_T(z)) \, dz = 1 - \frac{1}{2} \int |p_S(z) - p_T(z)| \, dz.
\]

Hence, the classification error simplifies to:
\[
\epsilon_{C^*} = \frac{1}{2} - \frac{1}{4} \int |p_S(z) - p_T(z)| \, dz.
\]

The total variation distance between the source and target latent distributions is defined as:
\[
\text{TV}(q_\phi(\mathcal{X}_S), q_\phi(\mathcal{X}_T)) = \frac{1}{2} \int |p_S(z) - p_T(z)| \, dz.
\]

Substituting into the expression for \( \epsilon_{C^*} \), we have:
\[
\text{TV}(q_\phi(\mathcal{X}_S), q_\phi(\mathcal{X}_T)) = 1 - 2\epsilon_{C^*}.
\]

The \(\mathcal{H}\)-divergence is related to the total variation distance by:
\[
d_{\mathcal{H}}(q_\phi(\mathcal{X}_S), q_\phi(\mathcal{X}_T)) = 2 \, \text{TV}(q_\phi(\mathcal{X}_S), q_\phi(\mathcal{X}_T)).
\]

Thus, substituting the total variation distance:
\[
d_{\mathcal{H}}(q_\phi(\mathcal{X}_S), q_\phi(\mathcal{X}_T)) = 2 \left(1 - 2\epsilon_{C^*}\right).\]
\end{proof}

\section{Implementation Details}
\label{sec:imp_details}
We set specific parameters for both the proposed architecture and baseline models to ensure fair comparisons. The number of iterations for all baseline topic models is fixed at 100, whereas DALTA converges within 20 epochs to achieve the reported performance in Section \ref{sec:experiment}. For contextualized document representations, we use the sentence transformer \textit{all-MiniLM-L6-v2}\footnote{\href{https://huggingface.co/sentence-transformers/all-MiniLM-L6-v2}{https://huggingface.co/sentence-transformers/all-MiniLM-L6-v2}}. Evaluation metrics are computed using the same parameter settings across all models; for example, the number of top words per topic for calculating $C_V$ and $TD$ is fixed at 10. In text classification experiments, we use the default parameters for SVC and LR from scikit-learn\footnote{\href{https://scikit-learn.org}{https://scikit-learn.org}}.

Our proposed model, DALTA, is built on a Variational Autoencoder (VAE) backbone with a latent dimension of 50. The number of topics is set to 50 for the source domain, while for the target domain, it varies depending on the experiments. To balance source and target domain contributions, DALTA employs a domain-weighted sampling strategy, controlled by a probability parameter $\mu$. Initially, $\mu = 0.7$, prioritizing source-domain samples for stable representation learning. As training progresses, $\mu$ gradually decreases to 0.3, shifting focus to the target domain for adaptation. This ensures the model effectively captures domain-invariant structures while retaining target-specific information through separate decoders.

\section{Computing Infrastructure}
\label{sec:compute}
All experiments were conducted on a server equipped with two AMD EPYC 7302 3GHz CPUs, three NVIDIA Ampere A40 GPUs (48GB VRAM each, 300W), and 256GB RAM.

\section{Case Study: Source Domain Selection via Internal Signals}
\label{sec:casestudy}
Although DALTA is not designed to perform source domain selection, we explore whether its internal training signals can be used as a practical heuristic to guide this choice in real-world low-resource scenarios. This is particularly relevant in situations where multiple candidate source domains are available, but selecting the most compatible one for a given target domain remains an open challenge.

We define a simple alignment score using two quantities readily available during the early stages of DALTA’s training: the \textit{domain alignment loss} $\mathcal{L}_{adv}$, which measures how well the source and target latent representations align, and the \textit{target reconstruction loss} $\mathcal{L}_{rec}^{(T)}$, which captures how accurately the model reconstructs documents in the target domain. The alignment score is defined as:
\[
\text{Alignment Score} = \mathcal{L}_{adv} - \lambda \cdot \mathcal{L}_{rec}^{(T)},
\]
where we set $\lambda = 0.001$ to balance the two components. We compute this after just 5 iterations of training—so it’s fast and doesn’t require full convergence. 

To evaluate the effectiveness of this heuristic, we conduct two case studies using the Newsgroup dataset. In the first case, we treat the NG \textsc{Science} subset as the target domain and consider four source domains: (i) the remaining portion of the 20 Newsgroup corpus excluding NG \textsc{Science}, (ii) AG News, (iii) Arxiv-CS abstracts, and (iv) Drug Review data. We observe that NG (excluding \textsc{Science}) yields the highest alignment score and also results in the best topic coherence and diversity. This suggests that even a partially overlapping source domain can provide valuable inductive bias for adaptation when its latent space aligns well with the target.

In the second case, we use NG \textsc{Religion} as the target domain and evaluate the same set of sources. AG News achieves the highest alignment score and also produces the best topic quality. While NG (excluding \textsc{Religion}) contains more topical overlap, its latent alignment with the target appears weaker, possibly due to vocabulary shifts or semantic granularity mismatches. In contrast, AG News contains general-purpose news content---including politics and society---that implicitly overlaps with religious discourse, leading to better adaptation.

These case studies suggest that DALTA's internal training dynamics can be leveraged to estimate source utility early in the learning process. Although this alignment score is heuristic and task-specific, it provides a promising starting point for developing lightweight, data-driven strategies for source domain selection in low-resource topic modeling.

\begin{table}[h!]
\centering
\resizebox{1.0\linewidth}{!}{%
\begin{tabular}{llccc}
\toprule
\textbf{Target} & \textbf{Source} & \textbf{DL} & \textbf{RL\textsubscript{T}} & \textbf{Score} \\
\midrule
\multirow{4}{*}{NG Science} 
  & NG (w/o Science) & 0.369 & 564.68 & \textbf{-0.196} \\
  & AG News          & 0.116 & 522.34 & -0.406 \\
  & Arxiv-CS         & 0.129 & 557.81 & -0.428 \\
  & Drug Review      & 0.095 & 524.99 & -0.430 \\
\midrule
\multirow{4}{*}{NG Religion} 
  & AG News          & 0.255 & 726.58 & \textbf{-0.472} \\
  & NG (w/o Religion)& 0.289 & 771.92 & -0.483 \\
  & Drug Review      & 0.208 & 732.50 & -0.525 \\
  & Arxiv-CS         & 0.110 & 765.89 & -0.656 \\
\bottomrule
\end{tabular}}
\caption{Alignment scores across source domains using DALTA's early training signals. Higher scores indicate better alignment and modeling suitability for the target domain.}
\label{tab:alignment-case-study}
\end{table}